\documentclass{article}
\usepackage[margin=34mm]{geometry}
\usepackage{tabularx}
\usepackage{algorithm}
\usepackage{algpseudocode}   
\usepackage{epic,eepic,textcomp,tikz,bbm}
\usepackage{amsmath,amsthm,amssymb}
\usepackage{enumerate}

\newtheorem{theorem}{Theorem}[section]
\newtheorem{assumption}[theorem]{Assumption}
\newtheorem{definition}[theorem]{Definition}
\newtheorem{remark}[theorem]{Remark}
\providecommand{\keywords}[1]
{
	\small	
	\textbf{\textit{Keywords---}} #1
}

\newcommand{\X}{\mathcal X}

\DeclareMathOperator*{\arginf}{arg\,inf}
\title{Training Generative Networks with Arbitrary Optimal Transport costs.}
\author{Vaios Laschos \\	 Fakult\"at Elektrotechnik und Informatik\\
	Technische Universit\"at Berlin\\ vaios.laschos@tu-berlin.de. \and 
	 Jan Tinapp \\ 	 Fakult\"at Elektrotechnik und Informatik\\
	 Technische Universit\"at Berlin\\jan.tinappp@tu-berlin.de.\and
	 Klaus Obermayer \\	 Fakult\"at Elektrotechnik und Informatik\\
	 Technische Universit\"at Berlin \\ klaus.obermayer@tu-berlin.de.
	 }
\begin{document}

\maketitle

\begin{abstract}
We propose a new algorithm that uses an auxiliary neural network to express the potential of the optimal transport map between two data distributions. In the sequel, we use the aforementioned map to train generative networks. Unlike WGANs, where the Euclidean distance is {\it implicitly} used, this new method allows to {\it explicitly} use {\it any} transportation cost function that can be chosen to match the problem at hand. For example, it allows to use the squared distance as a transportation cost function, giving rise to the Wasserstein-2 metric for probability distributions, which results in fast and stable gradient descends. It also allows to use image centered distances, like the structure similarity index, with  notable differences in the results.\newline
\end{abstract}


\section{Introduction}
In the seminal work of \cite{goodfellow2014generative}, a ground breaking  approach for training generative networks (GNs) was proposed. GNs try to generate new data given some training data $x$. This is achieved by treating the data as samples from an unknown empirical probability function $\mathbb{P}_r$ and fitting a model  $\mathbb{P}_{\theta}$ to give an estimate for this function. By drawing samples from $\mathbb{P}_{\theta},$ new data points can be generated that resemble the original dataset.
The approach in  \cite{goodfellow2014generative}, proposes to use two neural networks one of which is the {\it generator} $G(z)$ and the other acts as a {\it discriminator} $D(x)$. The {\it generator} represents the model  $\mathbb{P}_{\theta},$ while the {\it discriminator} gives a differentiable measure of how similar $\mathbb{P}_{\theta}$ is to $\mathbb{P}_r$. To train the {\it generator} the two networks participate in a two-player minimax game where the {\it discriminator} is maximizing the probability to detect if data shown to it comes from the {\it generator} or the training dataset. The {\it generator} is trying to fool the {\it discriminator} by minimizing the probability of generated data being detected. To train the networks  the value function $\displaystyle\min_G \displaystyle\max_D V(D,G) = \mathbb{E}_{x \sim \mathbb{P}_r}[\log D(x)]+\mathbb{E}_{x \sim \mathbb{P}_{\theta}}[\log(1- D(x))] $ is used by switching between taking the gradient with respect to the weights in $G(x)$ and $D(x).$  

As it is explained in  \cite{goodfellow2014generative}, the {\it discriminator} provides an error function for the {\it generator}. A fully trained {\it discriminator} will provide as a cost the relative entropy of the generated distribution $\mathbb{P}_{\theta}$  with respect to the real distribution $\mathbb{P}_r$. However relative entropy is a good metric distance for two distributions, only when they have the same support. If the distributions have different supports the error function cannot provide any information for the points charged by  $\mathbb{P}_{\theta}$ and not by $\mathbb{P}_r$. In practice this means that if the {\it discriminator} is well trained then it does not provide a gradient for the {\it generator} to learn, which result to the phenomenon of mode collapse. To fix that, researchers suggested various technical solutions, among them being the introduction of vanishing noise, since by adding some noise, the two  distributions have always the same support. However, most of these solutions were producing other problems in return. 

 \subsection{WGANs and beyond}
To deal with the issues that the original GANs exhibited, various new approaches were introduced, with the most noticeable being the Wasserstein GAN in \cite{Arjovsky2017}, that views the problem as an optimal transport task. In this approach the Wasserstein-1 distance between samples from the real distribution and samples from the current model is approximated and then minimized by changing the {\it generator} weights. In order to find the gradient to change the model, a differentiable way of calculating the Wasserstein-1 distance is required. This issue is resolved by introducing a neural network, the {\it critic}, that gets trained to approximate the Wasserstein distance between samples form both distributions.
 
To adapt the model with respect to the distance, the following formula was used: 
\begin{equation}\label{asa}
\begin{split}
W_{1}(\mathbb{P}_{\theta},\mathbb{P}_{r})&:=
\sup_{\psi\in\mathrm{Lip}_{1}}\bigg\{\int \psi\, d\mathbb{P}_{\theta}-\int\psi\, d\mathbb{P}_{r}\bigg\}.
\end{split}
\end{equation}
were $\psi$ is encoded by the {\it critic} and  $\|\psi\|_L \leq 1 $ constrains it to be 1-Lipschitz. 
Before we proceed, we would like to emphasize that unlike GANs, in WGANs the two networks do not act in an antagonistic fashion. It is possible, at least in theory, to fully train the  {\it critic} in every training step of the {\it generator}. Something like that is not advisable though, because it would be extremely expensive in computational time.

The simplicity for calculating the Wasserstein-1 distance comes from its dual formulation \eqref{asa}. One needs a \textbf{single function}, that can be encoded by the {\it critic} neural network, to express $\psi$ in \eqref{asa}, which in the sequel can be used to train the {\it generator}. At the same time, for different transportation costs, a nice formula like this is not known. It is important to note that, when one trains the {\it generator} with Euclidian distance as an underlying transport cost, the {\it generator} produces elements that are closer to the real ones with respect to that distance. Something like that may not be enticing when one wants to deal with \textbf{distributions of data that have a different intrinsic metric}. For example, when one tries to measure distance between two images, the Euclidean distance between pixels may not the be the right choice as Figure 1 indicates.
\begin{figure}[t!]
	\centering
	\includegraphics[ width=0.25\textwidth]{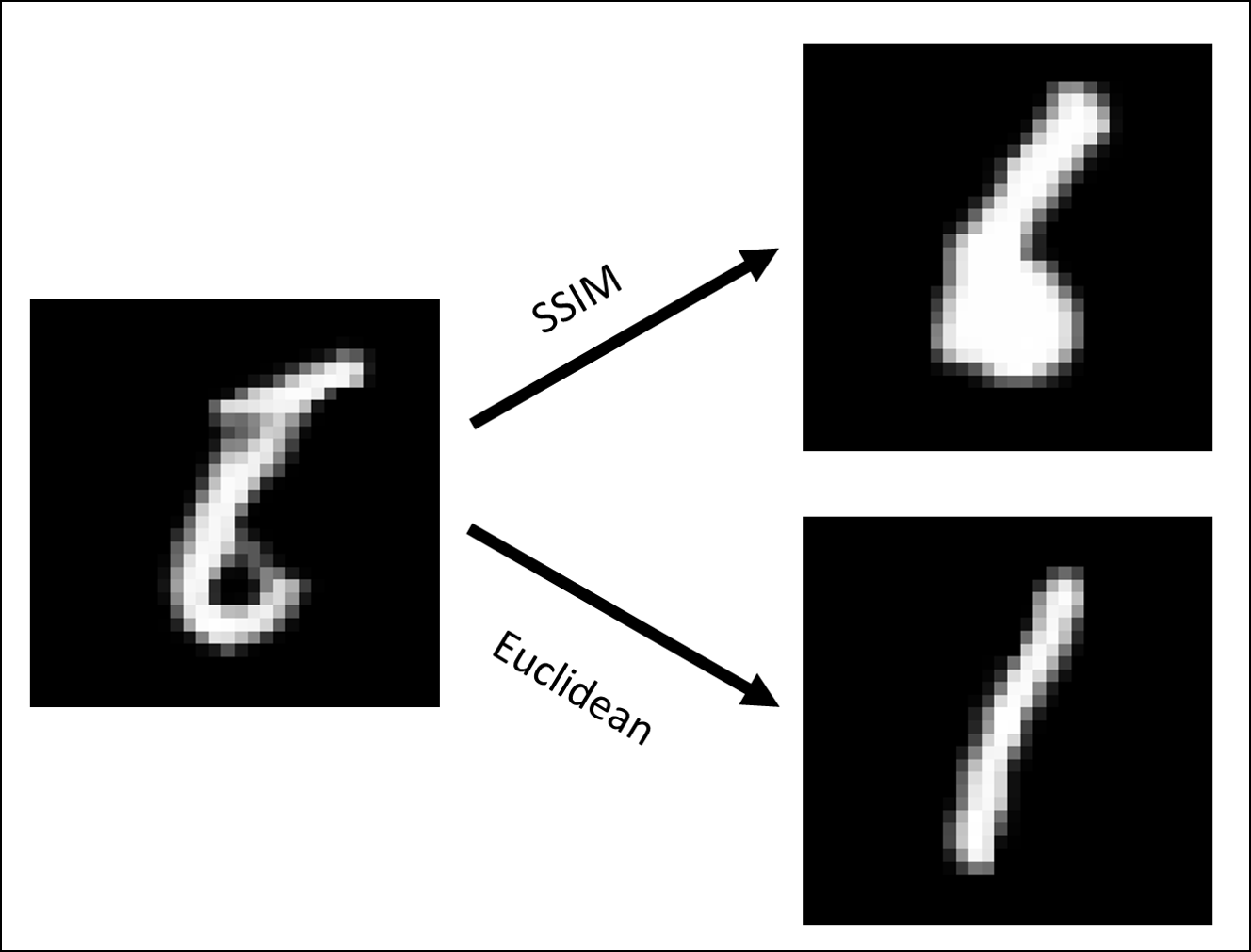}
	\caption{The closest real point to a generated one, in respect to SSIM and the Euclidean distance respectively.}
	\label{fig:RSUencountered}
\end{figure}
In this paper we propose a new algorithm for capturing the optimal transport distance between two distributions, with \textbf{arbitrary transport cost}. This way, the manifold that the {\it generator} produces could potentially fit better to the real data distribution. Also as we will argue in the sequel, our proposed method allows to optimize the training to any desirable degree, restricted only by processing power, and it can ensure that no mode collapse occurs. 

\subsection{A new approach}

  Regardless of any mathematical complexity that the derivation of our training algorithm may exhibit, its heuristic explanation can be easily grasped. We start by taking two sets of points $\X_{1}, \X_{2}.$ For visualization purposes one can think $\X_{1}$ as the set of generated points and $\X_{2}$ as the set of real points.
 For each point $x$ in $\X_{1},$  we assign one point in set $\X_{2},$ through the formula \begin{equation}\label{houhou}
 y(x,w)=\arginf_{y\in \X_{2}}\{c(x,y)+\psi_{w}(y)\},
 \end{equation}
 where $c:\X_{1}\times\X_{2},$ is some optimal transport cost and $\psi_{w}$ a function parametrized by the ``weight values'' $w.$  Now, some real points in $\X_{2}$ have many generated points in $X_{1}$ assigned to them, while at the same time, some others have no points assigned to them at all. For the real points that have many generated points assigned to them, we would like $\psi_{w}$ to increase, and for those who are under-assigned to decrease.  It is shown that the {\it assigner} $\psi_{w}$ is trained to perfection, when all points are assigned equally, and in that case, $\psi_{w}$ coincides with  the potential (``pregradient") of the optimal transport plan.  We achieve that by using
  \begin{equation}\label{hoxi}\sum_{j=1}^{\#\X_{1}}\left(\frac{\#\{\{x_{i}, 0\leq i\leq \#\X_{2}:y(x_{i},w_{0})=y_{j}\}}{\#\X_{2}}-\frac{1}{\#\X_{1}}\right)\psi_{w}(y_{j}),\end{equation}
  as a cost to our auxilary network. We observe that the change to the values of $\psi_{w}$ in a real point is weighted by the numbers of generated points assigned to it. Finally, the constant $\frac{1}{\#\X_{1}}$ that appears in the error function,  can be understood as a normalization constant.
\begin{figure}
 	\centering
 	\begin{tabular}[b]{c}
 		\includegraphics[width=0.40\textwidth]{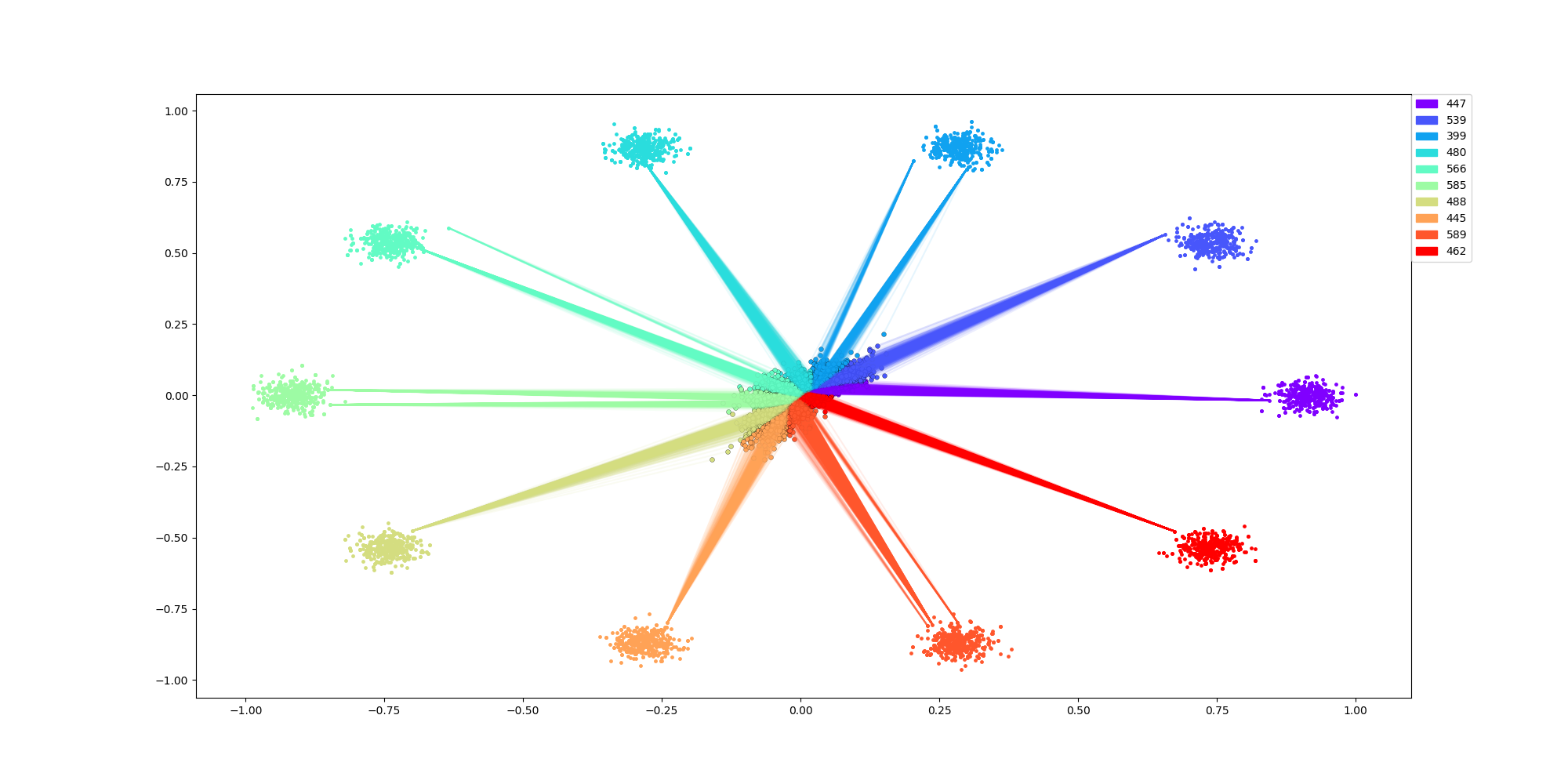}\\
 		\small{Assignments before training the {\it generator}}
 	\end{tabular}
 	\begin{tabular}[b]{c}
 		\includegraphics[width=0.40\textwidth]{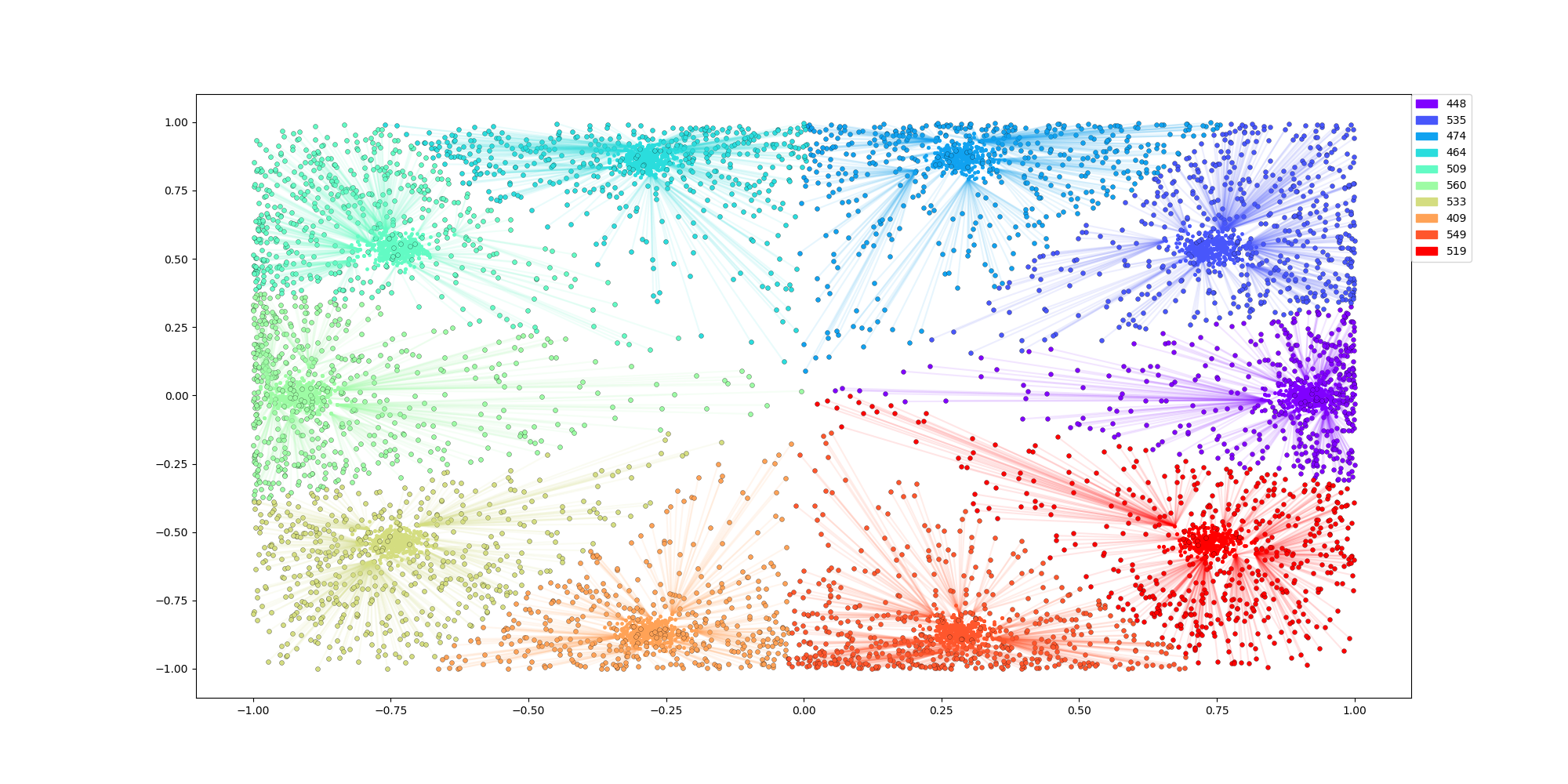}\\
 		\small{After the first training of the {\it generator}}
 	\end{tabular}
 	\begin{tabular}[b]{c}
 		\includegraphics[width=0.40\textwidth]{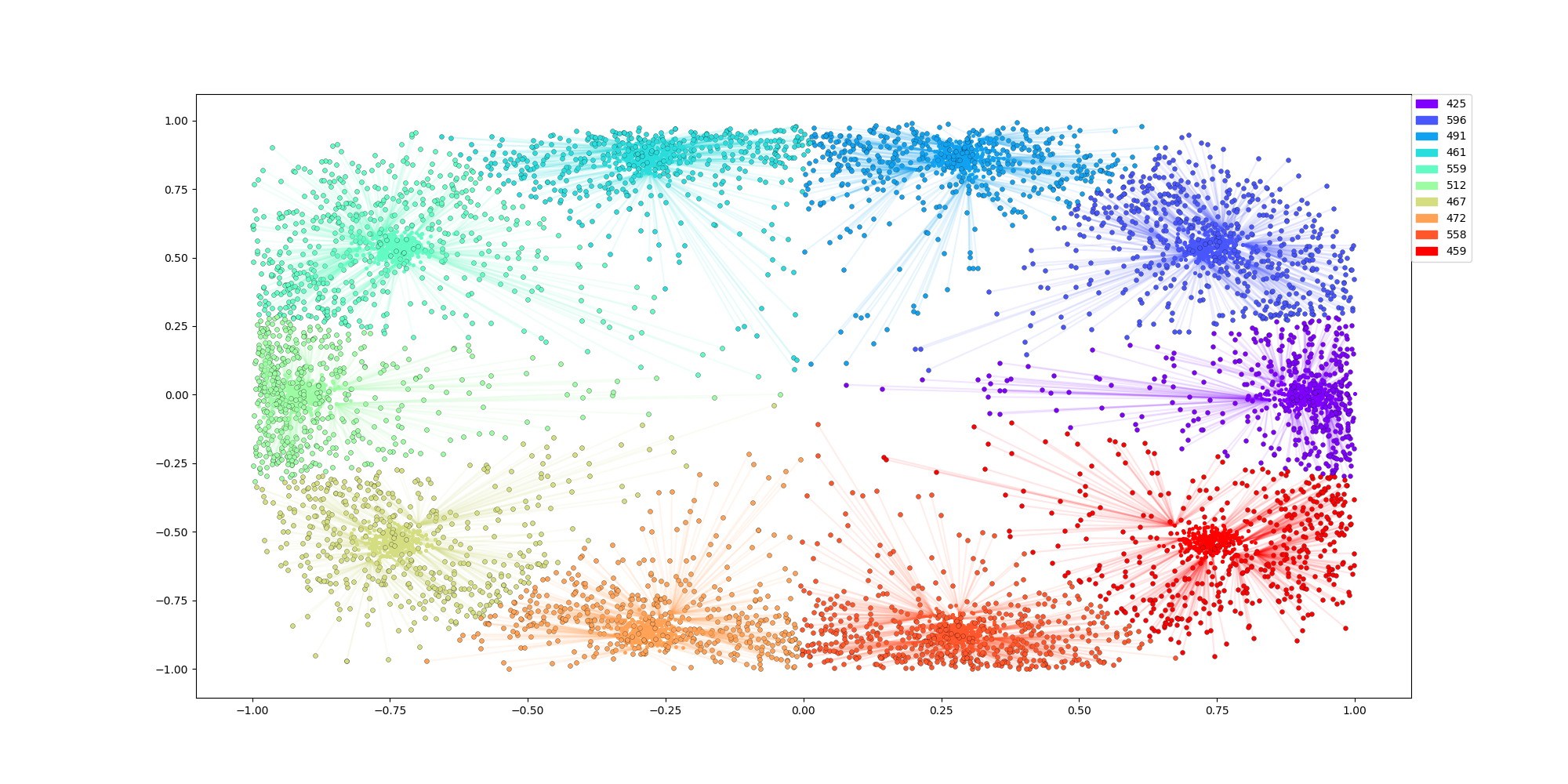}\\
 		\small{Ten steps for {\it generator}}
 	\end{tabular}
 	\begin{tabular}[b]{c}
 		\includegraphics[width=0.4\textwidth]{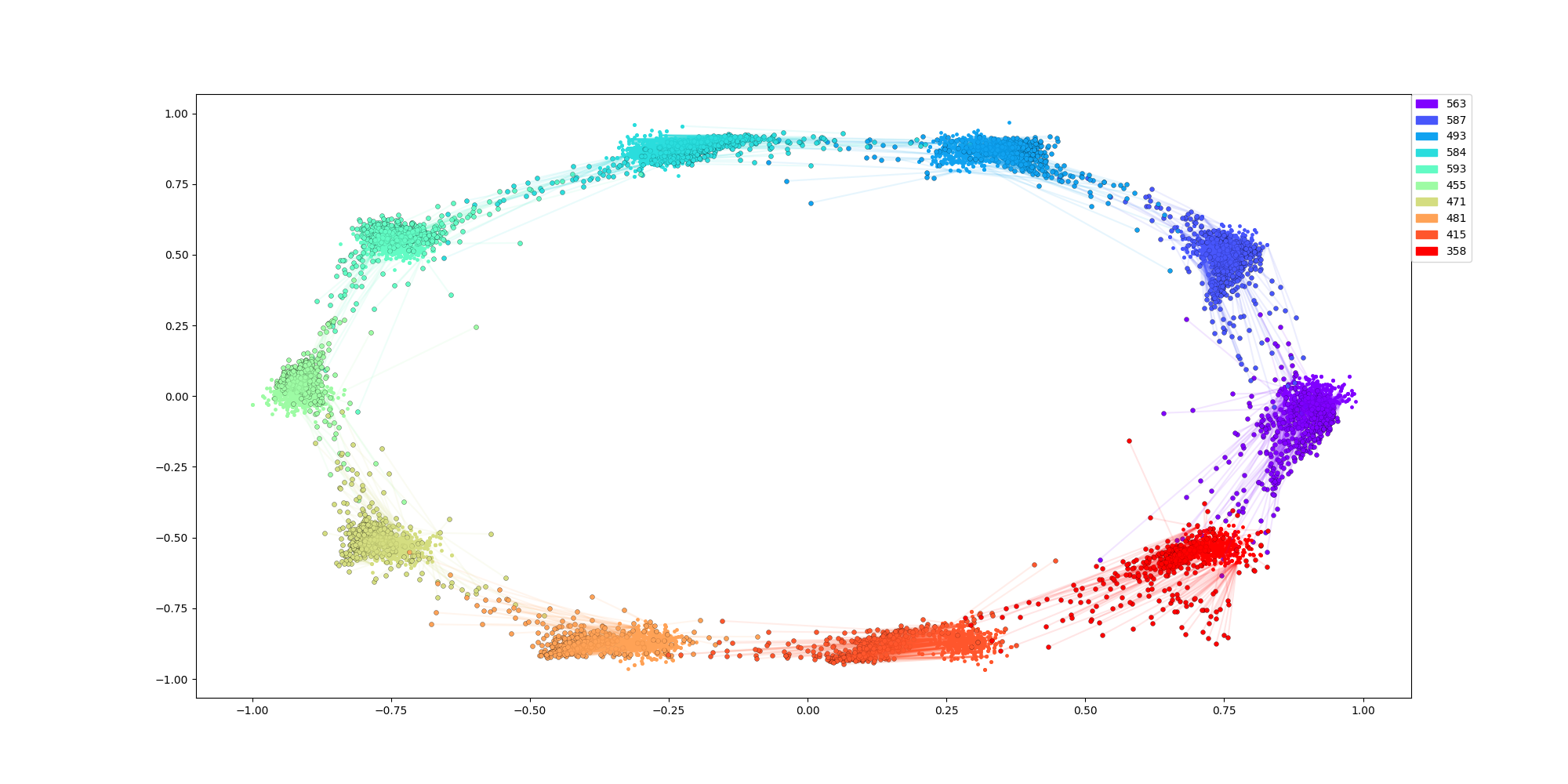}\\
 		\small{Twenty steps for the {\it generator}}
 	\end{tabular}
 	\caption{Visualization of the assignments of generated points, initially found in the middle, to a ring of gaussians.}\label{fig:assignment_training_gaussian}
 	\label{figgg}
 \end{figure}
 When the {\it assigner} is trained to perfection, then one can use the assignments to send the generated points to the reals that are assigned. As a visual aid, one can look at Figure \ref{figgg}, where we treat the case of the real data being 2000 points generated by a ring of 10 Gaussians. \newline

 \textbf{First contribution:} A proof of concept that one can use dual formulations for arbitrary optimal transport problems, i.e.
 \begin{equation}\label{be}
 \sup_{\psi\in C_{b}(\X)}\bigg\{\int \psi^{supp(\nu),c}\, d\mu(x){-}\int\psi\, d\nu(y)\,\Big|\, \psi^{supp(\nu),c}(x)=\inf_{y\in supp(\nu)}\{c(x,y)+\psi(y)\}\bigg\},
 \end{equation}
 to recover an applicable algorithm for training generative networks with good results, and not only when the cost is the euclidean distance and the resulting metric is Wasserstein-1. It is reasonable to assume, that a generator trained with a method like that and with a cost functions that captures the geometry of the original dataset, will also reproduce that geometry.\newline

\textbf{Second contribution:} We rigorously show, that from a probabilistic point of view, one can exchange the gradient with the $\arginf$ operator  in \eqref{houhou}, eventually reaching formula \eqref{hoxi}. This makes the whole \emph{dual formulation approach} computationally feasible. Calculating the gradient of \eqref{houhou} instead, would require an infeasible  number of statistically useless computations. Furthermore, in formula \eqref{houhou}, the {\it critic} is not evaluated for the generated points. Since the generated points do not directly occur in the error function as in \eqref{asa}, any number of them can be generated, in order to achieve any desired proximity to the full theoretical error function \eqref{be}. On the other hand, all the real points contribute to the value of the error function, posing this way a restriction on the size of the original dataset, which is the main drawback of our approach.
 
 \subsection{Further comments and outline}
 In the vast literature of GANs and WGANs, there are various papers that claim the use of general metric distances. Among the most noticeable are \cite{Liu} and \cite{Avraham}, that produce state of the art results. However, in these papers a regularization term is added to a WGAN training cost to achieve some of the qualities of the transport cost that they claim to apply. Although really difficult to properly formulate and prove, we expect that \textbf{no WGAN-like method, that uses a single function} $\psi$ can capture general optimal transportation in its entirety. This happens because many of the literariness that the Wasserstein-1 formulas enjoy,  break in cases where the cost function is not a metric distance. \textbf{We would like to emphasize at this point, that our algorithm, regardless of any criticism that may apply to it, does exactly that.} It trains in a matter that real optimal transportation with general cost function, between the two distributions, is achieved. \textbf{It does not train like other GANs/WGANs and it is not a variation of them.}
 
 We would finally like to mention, that although we run some comparative experiments with a good outcome, we restricted our experiments to very basic datasets like MNIST and FASHION-MNIST. We lacked both the expertise and the computational power to run more sophisticated tests with more demanding datasets. We do not claim our algorithm to be state of the art, but it acts as a proof of concept, and we aim to improve in various directions in the future. Furthermore we used the real Wasserstein distance, measured by an external application, between the whole real and generated distributions to show that at least from a mathematical perspective, our method minimizes better the distances that claim to capture.

The layout of the paper is the following. First we have a discussion on optimal transportation distances, and a short literature review on the special case of the Wasserstein-1 distance. We proceed with a mathematical analysis of our method, to which we will refer from now on as the {\it assignment method}. Next, we provide a heuristic description of the difference in training between GANs and the assignment method. Finally we provide our experimental results and we enumerate the strengths and drawbacks of the assignment method.

\section{Optimal transportation}
For the sequel, let $(\X,d_\X)$ be a  compact metric space. We will
denote by $\mathfrak{M}(\X)$ the space of all 
nonnegative and finite Borel measures on $\X$ endowed with the 
weak topology induced by the duality with the continuous and bounded
functions of $\mathrm{C}_b(\X)$. The set $\mathcal{P}(\X)\subset\mathfrak{M}(\X)$ is the subset of probability measures. Let finally $c:\X\times \X\to\mathbb{R}^{+},$ a transportation ``cost" function that is continuous with respect to the distance $d_\X.$ For two measures $\mu, \nu,$ a transport plan between $\mu$ and $\nu$ is a measure $Q\in \mathcal{P}(\X{\times} \X)$ with marginals $\mu:=\pi^0_\sharp Q, \nu:=\pi^1_\sharp Q$. In the last line, we applied the following definition.

\begin{definition}
If $\mu\in \mathcal{P}(\X)$ and $T:\X\to \mathcal{Y}$ is a Borel map, $T_\sharp \mu$ will denote the
push-forward measure on $\mathcal{Y}$, defined by
\begin{equation}\label{eq:push_forw}
T_\sharp\mu(B):=\mu(T^{-1}(B))\quad
\text{for every Borel set $B\subset Y$}.
\end{equation}
\end{definition}

Now we define the optimal transport distance between two measures $\mu,\nu.$

\begin{definition}
Given a couple of measures $\mu,\nu\in \mathcal{P}(\X),$  their
$c$-Optimal Transportation distance $\mathcal{T}_{c}$ is defined by
\begin{equation}
\mathcal{T}_{c}(\mu,\nu):=\min_{Q\in\mathcal{P}(\X{\times} \X)}\bigg\{\iint 
c(x,y)\, d Q(x,y)\,\Big|\,  \pi^0_\sharp Q=\mu \wedge\ \pi^1_\sharp Q=\nu\bigg\}.
\end{equation}
\end{definition}

For the case where $c=d_{\X},$ we recover the so called Monge-Kantorovich distance or Wasserstein-1, i.e. $W_{1}=\mathcal{T}_{d_{\X}}$. Also when $c=d^{2}_{\X},$ the square root of $\mathcal{T}_{d^{2}_{\X}},$ gives the Wasserstein-2 distance, i.e. $W_{2}=\sqrt{\mathcal{T}_{d^{2}_{\X}}}.$ In the sequel, we are going to denote the set of all integrable functions with respect to some probability measure $\mu,$ with $\mathcal{L}(\mu).$ We are also going to use the notation $supp(\mu)$ for the support of $\mu.$ From \cite{Villani2008}, we have:

\begin{theorem}[Dual formulation]
For $\mu,\nu\in\mathcal{P}(\X),$ we have
\begin{equation}\label{dual}
\begin{split}
\mathcal{T}_{c}(\mu,\nu):&{=}\!\!\sup_{\phi,\psi\in C_{b}(\X)}\!\!\bigg\{\int \phi\, d\mu(x)-\int\psi\, d\nu(y)\,\Big|\, \phi(x)+\psi(y)\leq c(x,y)\bigg\}\\&{=}\!\!
\sup_{\psi\in C_{b}(\X)}\!\!\bigg\{\int \psi^{c}\, d\mu(x)-\int\psi\, d\nu(y)\,\Big|\, \psi^{c}(x)=\inf_{y\in \X}\{c(x,y)+\psi(y)\}\bigg\}\\&{=}\!\!\!\!
\sup_{\psi\in C_{b}(\X)}\!\!\bigg\{\!\!\int\! \psi^{supp(\nu),c}\, d\mu(x){-}\int\!\psi\, d\nu(y)\,\!\Big|\!\, \psi^{supp(\nu),c}(x)=\!\!\!\!\!\inf_{y\in supp(\nu)}\!\!\!\!\{c(x,y)+\psi(y)\}\!\!\bigg\}.
\end{split}
\end{equation}
\end{theorem}	
\subsection{The case of the Wasserstein-1 distance: from weight clipping to the two-step approximation method}
In \cite{Arjovsky2017}, the authors suggest the use of the Wasserstein-1 distance as an error function for training GANs. In order to implement the Wasserstein-1 distance,
the authors used a special form of the dual formulation  \eqref{dual} that holds only for the case of the Wasserstein-1 distance. Specifically, they used the fact that if $\psi$ is a Lipschitz function with Lipschitz constant smaller than one, then $\psi^c$ is equal to $\psi.$ This gives rise to the formula \begin{equation}\label{dualC}
\begin{split}
W_{1}(\mu,\nu)&:=
\sup_{\psi\in\mathrm{Lip}_{1}}\bigg\{\int \psi\, d\mu(x)-\int\psi\, d\nu(y)\bigg\}.
\end{split}
\end{equation}

In order to implement this form of the distance, one had to establish that during the training procedure the function $\psi$ has to have Lipschitz constant of less than one. In order to achieve that, they introduced the method known as {\it weight clipping}, where any time a weight exceeds a specific limit, it is reduced so f can retain its Lipschitz constant. However, this first approach, although quite plausible, did not have any rigorous mathematical justification, and it is known to cause issues like stagnation or instability of the learning process, as it was pointed out in \cite{Gulrajani2017a}. Following \cite{Arjovsky2017}, in several publications, authors tried to fix, improve or replace this method. Among the most notable papers, in \cite{Gulrajani2017a} the authors propose a method to stabilize the gradient. More recently, in \cite{Liu2018} the authors propose to first solve a linear programming problem and then use the solution to train the {\it critic} to mimic the solution. Evenmore the method \cite{Liu2018} allows for more general transportation costs, provided that they satisfy the triangular inequality. However, even in the method proposed in \cite{Liu2018}, the squared distance is not applicable, since it does not satisfy the triangular inequality. Furthermore, this method requires a process that is outside the training circle, making it difficult to compare efficiency with established methods. We remind once more, that there are papers with state of the art results like \cite{Liu} and \cite{Avraham}, where it is claimed that more general optimal costs are reproduced. However, in these papers a regularization term is added to a WGAN training cost to achieve, only in part, the qualities of the transport cost that they claim to apply.

\section{Mathematical justification for a new network type and error function: The Assignment method}\label{sec:math}
Before we proceed, we will make the following assumption for the cost $c,$ that it holds true for all norms.
\begin{assumption}\label{assumption}
For every $x\in\X,$ the level sets of $c(x,\cdot),$ ie $\X_{x,a}=\{y\in\X:c(x,y)=a\},$ have Lebesgue measure zero.
\end{assumption}

Let now $\psi_{w}$ denote the function that corresponds to the assigning network  (imagine that as the analogous of a {\it critic}) with weights $w.$  We will also assume that $\nu=\sum_{j=1}^{M}\delta_{y_{j}}$ is the distribution of all real points, and $\mu={\mathbb{P}}_{\theta}$ is the distribution of $G_{\theta}(z),$ where $G_{\theta}$ is the {\it generator}, $z \sim p,$ and $p$ is the ``noise" latent distribution. We independently pick a sequence of points $x_{i}$ from ${\mathbb{P}}_{\theta}.$ We denote with $\mu_{N}=\sum_{i=1}^{N}\delta_{x_{i}}$ the \textbf{Nth-order empirical distribution} by summing the first $N$ points. Finally for every generated point $x,$ we define $y(x,w),$ as follows:
\begin{equation}
y(x,w)=\arginf_{y\in supp(\nu)}\{c(x,y)+\psi_{w}(y)\}.
\end{equation}
For fixed $(x,w),$ $y(x,w)$ is the point that is assigned by the formula for the dual. As we show in the appendix, this point is unique with probability one. If one point does not have a unique assignment then we arbitrary assign a point without loss of generality. By standard results in probability, we have that almost surely it holds $\mu_{N}\rightarrow\mu,$ and therefore

\begin{equation}\label{haha}
\begin{split}
&\mathbf{D}_{w}\left(\int \psi_{w}^{supp(\nu),c}(x) d\mu(x)-\int\psi_{w}(y)\, d\nu(y)\right)_{|_{w=w_{0}}}\\
&=\int \mathbf{D}_{w}\left(\psi_{w}^{supp(\nu),c}(x)\right)_{|_{w=w_{0}}} d\mu(x)-\int\mathbf{D}_{w}\left(\psi_{w}(y)\right)_{|_{w=w_{0}}}\, d\nu(y)\\&
=\lim_{N\rightarrow\infty}\int \mathbf{D}_{w}\left(\psi_{w}^{supp(\nu),c}(x)\right)_{|_{w=w_{0}}} d\mu_{N}(x)-\int\mathbf{D}_{w}\left(\psi_{w}(y)\right)_{|_{w=w_{0}}}\, d\nu(y)\\&=\lim_{N\rightarrow\infty}\mathbf{D}_{w}\left(\int \left(\psi_{w}(y(x,w))+c(x,y(x,w))\right) d\mu_{N}(x)-\int\psi_{w}(y) d\nu(y)\right)_{\!\!|_{w=w_{0}}}
\\&=\lim_{N\rightarrow\infty}\mathbf{D}_{w}\left(\frac{1}{N}\sum_{i=1}^{N}\psi_{w}(y(x_{i},w)) -\frac{1}{M}\sum_{j=1}^{M}\psi_{w}(y_{j})+\frac{1}{N}\sum_{i=1}^{N}c(x_{i},y(x_{i},w))\right)_{\!\!|_{w=w_{0}}},
\end{split}
\end{equation}

where $\mathbf{D}_{w}$ is the derivative with respect to $w,$ in the first and third equality we applied Leibniz's rule, and in the second equality we used the fact that $\mu_{N}\rightarrow\mu.$ Since $w_{0},$ is fixed, with probability one, we get that the last term in \eqref{haha} is equal  to

\begin{equation}\label{aha}
\begin{split}
&\lim_{N\rightarrow\infty}\mathbf{D}_{w}\left(\frac{1}{N}\sum_{i=1}^{N}\psi_{w}(y(x_{i},w_{0})) -\frac{1}{M}\sum_{j=1}^{M}\psi_{w}(y_{j})+\frac{1}{N}\sum_{i=1}^{N}c(x_{i},y(x_{i},w_{0}))\right)_{\!\!|_{w=w_{0}}}\\
&=\lim_{N\rightarrow\infty}\!\mathbf{D}_{\!w}\!\!\left(\!\frac{1}{N}\sum_{i=1}^{N}\psi_{w}(y(x_{i},w_{0})) -\frac{1}{M}\sum_{j=1}^{M}\psi_{w}(y_{j})\!\right)_{\!\!|_{w=w_{0}}}\\&=\lim_{N\rightarrow\infty}\sum_{j=1}^{M}\left(\frac{\#\{\{x_{i}, 0\leq i\leq N:y(x_{i},w_{0})=y_{j}\}}{N}-\frac{1}{M}\right)\mathbf{D}_{w}\left(\psi_{w}(y_{j})\right)_{|_{w=w_{0}}},
\end{split}
\end{equation}

where in the \eqref{aha} one can notice that $y$ does not vary with $w$ anymore, and that is why the term $\mathbf{D}_{w}\frac{1}{N}\sum_{i=1}^{N}c(x_{i},y(x_{i},w))$ vanishes. This happens because, for small changes in $w,$ the assignments do not change with probability one. The proof of that claim can be found in the appendix.

 This proves that the error functions 
 \begin{equation}\label{oror}
 \int \psi_{w}^{supp(\nu),c}(x) d\mu(x)-\int\psi_{w}(y)\, d\nu(y)
 \end{equation} and 
 
 \begin{equation}\label{Asscost}\sum_{j=1}^{M}\left(\frac{\#\{\{x_{i}, 0\leq i\leq N:y(x_{i},w_{0})=y_{j}\}}{N}-\frac{1}{M}\right)\psi_{w}(y_{j}),\end{equation} where $N$ is picked sufficient big, can be used interchangeably for the ``critic". We note that, by applying the theory of Large Deviations (see \cite{Dembo1998}), one can prove that the probability of the distance between the gradients of \eqref{oror} and \eqref{Asscost}, being bigger than some $\epsilon,$ decays exponential with the number of $N.$
 
The cost appearing in \eqref{Asscost}, will be referred from now on as the assignment cost, and it is the one that we used to train our auxiliary network (``{\it assigner}''). In the next section we are going to further analyze the idea behind the assignment cost and compare it with the {\it critic} cost in WGANs. We conclude with what we believe to be an interesting remark.

 \begin{remark}
 	By taking the derivative of \eqref{Asscost} one gets $$\sum_{j=1}^{M}\left(\frac{\#\{\{x_{i}, 0\leq i\leq N:y(x_{i},w_{0})=y_{j}\}}{N}-\frac{1}{M}\right)\mathbf{D}_{w}\left(\psi_{w}(y_{j})\right)_{|_{w=w_{0}}}.$$ If all points $x_{i}$ are assigned equally to every $y,$ then the derivative is equal to zero and therefore the training halts. Although this would be sufficient for the {\it generator} to be trained, one can wonder if we end up with an optimal assignment between $\mu_{N}$ and $\nu.$ Trying to answer this question we came up with a theorem that confirms that.	
 \end{remark}
More specifically we have

\begin{theorem}\label{minith}
Let $\mu,\nu\in\mathcal{P(\X)},$ $\psi\in \mathcal{L}(\mu).$ We further assume that there exists a unique minimizer $\widetilde{T}(x)$ 
of $\{c(x,y)+\psi(y)\},$ for $\mu$ almost every $x.$ Then if $\widetilde{T}_{\sharp}\mu=\nu,$ we have that $(I\times\widetilde{T})_{\sharp}\mu$ is an optimal plan. 
\end{theorem}

We note that, by Theorem \ref{one}, if $\mu$ is an absolutely continuous measure, $\nu$ is purely atomic, and $c$ satisfies Assumption \ref{assumption} then the map $\widetilde{T}$ is always defined. To reformulate, in order to know if $\psi$ is a maximizer in the dual formulation for $\mu,\nu,$ we only have to check that  $\{c(x,y)+\psi(y)\},$ has a unique minimizer $\widetilde{T}(x)$ for $\mu$ almost every $x,$ and that minimizer satisfies $\widetilde{T}_{\sharp}\mu=\nu.$

\section{Heuristic comparison between Assignment training and WGANs}
Before we proceed with the  comparison between WGANs and our method, we would like to share a little bit from the history of our research for educational purposes. We believe that the following will help new researchers, especially those coming from a pure mathematics background, clarify some things about how WGANs work, and avoid our mistakes.

We start by, once more, reminding the reader about the dual formulation of the transport cost,  i.e.
\begin{equation}\label{reminder}
\begin{split}
\mathcal{T}_{c}(\mu,\nu){=}\!\!\!\!
\sup_{\psi\in C_{b}(\X)}\!\!\bigg\{\!\!\int \psi^{supp(\nu),c}\, d\mu(x)-\int\psi\, d\nu(y)\,\!\Big|\!\, \psi^{supp(\nu),c}(x)=\!\!\!\!\inf_{y\in supp(\nu)}\!\!\!\{c(x,y)+\psi(y)\}\!\!\bigg\}.
\end{split}
\end{equation}
Note that, in order to calculate the dual function of $\psi$, one has to go through \textbf{only} the points that are in the $supp(\nu)$ and not through the whole space $\X$. Our initial approach was to use small batches, in the same way that WGANs are traditionally trained, and with the hope that if the batches come closer then the full distributions of real and generated points will also come closer together. We note that in order to calculate the infimum  in a differentiable way we applied a smooth maximum by using the LogSumExp function. Although this approach worked fine with low-dimension datasets having only a few nodes, it failed with datasets like MNIST. What we observed there was the production of blurred idealized versions of the digit. 

Further experiments showed that an increase in the size of the batches for training both {\it generator} and {\it critic} positively increased the image quality. This lead us to believe that approximating the distance will only work when the number of samples is really high. When thinking of the problem as an optimal transport problem this becomes much clearer, as there is the  possibility that samples from the closest manifold inside the data might not be in the batch we use for approximation. Furthermore if we think of each image from the MNIST dataset as a point in dimension $\mathbb{R}^{28x28}$ of a probability distribution we would need significantly more samples to accurately capture a distribution or even find a sufficient approximation. As it was pointed out in \cite{Arora} and \cite{Arora2017}, this line of reasoning, i.e. if the batches come closer then the full distributions of real and generated points will also come closer together, for why WGANs {\it do} work, is not valid anyway. By applying a simple mass concentration argument, they show, that in order for this argument to be valid, one needs to increase the batch size exponentially with the number of dimensions. Something like that is of course impossible in practice.

In order to understand why the WGAN method works with small batches, we came with the following heuristic explanation. When the {\it critic} is fed with some real points, its value value around these points increases, and when it is fed with some generated points, the value around them decreases. This is easy to understand by carefully dissecting the error function
\begin{equation}\label{mm}
\sum_{x_{i} \sim \mathbb{P}_r}\psi(x_{i})  - \sum_{y_{j} \sim \mathbb{P}_\theta}\psi(y_{j}). 
\end{equation}
 Then, roughly speaking, the trained {\it critic} is encoding a landscape where the generated points are valleys and the real ones are hills. The {\it generator} follows that landscape to "roll down" the generated points to the reals. The idea of the gradient penalty, apart from its mathematical justification, enforces this explanation scheme, because in practice, it smoothens the landscape along the lines between real and fakes. Therefore we hypothesized that it is not the fact that the batches really capture the two distributions, which makes WGANs work, but that:
 \begin{itemize}
 	\item When the learning rate is small enough, then it makes no practical difference between applying \eqref{mm} for many consecutive small batches or for a really big one. We believe that this is due to 
 	\begin{enumerate}
 		\item the linearity of the cost
		\item the fact that the {\it critic} network has enough degrees of freedom such that local changes do not significantly affect the rest of the network. 
\end{enumerate} 
 	\item Alternating between training {\it critic} and {\it generator}, is crucial for the generated points to go closer to the real ones.
 \end{itemize} 

To test this perception, that it is not the {\it critic} cost \eqref{mm} that in every step really captures the distance but the fact that when the learning rate is small enough, then it makes no practical difference between training the {\it critic} in many consecutive small batches or in a really big one, we trained WGANs with gradient penalty and with really small batches of 2 or 3 points. Given enough time, the result was almost as good as training with a batch of 64 or 128. At the same time, we also tried to train a perfect {\it critic} first and then train the {\it generator}, and this method failed even after thousand iterations of the {\it critic}.

Now, if one wants to train with a general transportation cost using \eqref{reminder}, this is not longer possible. In order for the dual of the {\it critic} to be defined properly then one has to go through a set of reals that capture well its distribution. If one tries to apply smaller batches the training fails. For visual purposes, we would like to note that unlike in WGANs, in the Assignment method, the {\it assigner} never goes through the generated points. Furthermore, the {\it assigner} does not create a landscape where real points are high and generated points are low so the {\it generator} can use to train. In the assignment method, the {\it assigner} increases at real points which are assigned to too many generated points and decreases otherwise. When the {\it assigner} is trained to optimality then the training of the {\it generator} happens through the assignment and the cost that shapes the {\it assigner}.

\section{Psudocode and comparison with WGANs}

\begin{algorithm}
\caption{WGAN2. We use the parameters  $\alpha=0.00005$, $m=64$, $n_{{\it critic}}=5.$}\label{code:assignment}\renewcommand{\thealgorithm}{}
\floatname{algorithm}{}
\begin{algorithmic}[1]
		\Require $\alpha$ is the learning rate. $n_{{\it critic}}$, the number of {\it assigner} iterations per {\it generator} iteration. $m$, the number of assignments per iteration of the {\it assigner}.
		\Require $w_0$, initial {\it assigner} parameters. $\theta_{0}$, initial {\it generator} parameters. $\X$, Matrix containing all of the real samples.  
\While {$\theta$ has not converged}\label{conver}
	\For { $t=0,...,n_{{\it critic}}$ } 
		\For {$i=0,...,m$}\label{batching1}
			\State Sample latent space  $z \sim p(z)$ \label{latent}
			\State $K^{i} \leftarrow \left(X[argmin(A_{w}(\X) + cost(\X,G_\theta(z))], G_\theta(z)\right)$\label{c:cost}
			\State $L^{i} \leftarrow - A_{w}(K^{i}(0))$
		\EndFor
		\State $w \leftarrow RMSProp \left(\mathbf{D}_w \left(\frac{1}{m} \sum_{i=1}^m  L^{i} -\frac{1}{len(\X)}\sum_{x\in\X}A_{w}(x)\right) ,w , \alpha \right)$
	\EndFor
	\For {$i=0,...,m$}\label{batching2}
		\State $L^{i} \leftarrow  \widetilde{cost}(K^{i}(0),K^{i}(1))$ \label{c:cost_tilde}
	\EndFor
	\State $\theta \leftarrow RMSProp(\mathbf{D}_{\theta} \frac{1}{h} \sum_{i=1}^h  L^{i} ,w , \alpha)$
\EndWhile
	 \end{algorithmic}\end{algorithm}

We have the following comments regarding our psudocode
\begin{itemize}
	\item When $m$ is relatively small, the whole training behavior resembles this of WGANs. However, one can use a different approach with the choice of $m$. If $m$ is chosen to be at the same scale with the size of $\mathcal{X},$ then the real distance between real and generated points can be captured, and evenmore achieve prevention of mode collapse. More specifcially, experimenting with various datasets of up to 10000 points, we noticed that 10 times the dataset will suffice. We expect that by some probabilistic arguments, the appropriate size of generate points can be estimated. Regardless of the size of $m,$ the generated sets can be batched.
	\item  Line \ref{latent}. Traditionally the latent distribution $p(z)$ is chosen to be a Gaussian. However we noticed that if we instead choose a collection of Gaussians with small variance the results are much better. 
	\item  Line  \ref{c:cost} and \ref{c:cost_tilde}: $\widetilde{cost}$  and $cost$ have to coincide from a theoretical point of view. However, in practice, a computationally faster cost can be used for the {\it {\it assigner}}.  For example, SSIM is computationally very demanding, and by using peak signal-to-noise ratio (PSNR) for the critic instead, one can get similar results in a small portion of the time. It also appears that the {\it assigner} trains the fastest, if we multiply the cost with a constant such that the diameter of the space is equal to one. This does not change the geometry of the fitted model.
\end{itemize} 
\subsection{Practical comparison between Assignment method and WGANs.}
\textbf{Advantages of Assignment method over WGANs.}
\begin{enumerate}
	\item Since, at every step, an actual estimate for the transport distance can be calculated (when the generated batch is big enough), the method provides a quantitative way to bound from above the full transport distance between the distributions. As a byproduct, this allows in only a few iterations, to check how different parameters affect the training. For example, we noticed that with Fashion MNIST one gets the best approximation when latent dimension is around 250. Contrary, with the classical MNIST dataset, no significant difference appear if we increase the latent dimension further than 100.
	\item We can approach the original distribution in any desirable degree. Furthermore we can ensure that no mode collapse occurs.
	\item Depending on the complexity of the generated set and cost function, it is possible for the {\it assigner} to be trained to optimality by itself and without any iterations with the {\it generator}. This way, we can calculate the transport distance between the two measures, and retrieve the optimal transport map.  
	\item The method does not not depend on the architecture of the network. It can work really well with the simple dense networks unlike the traditional WGANs where the assistance from using convolutional networks is noticeable.
\end{enumerate}

\textbf{Disadvantages of Assignment method over GANs and WGANs.}

\begin{enumerate}
	\item In its current version, the assignment algorithm requires to go through all the real points every time we generate a new point. This makes it quite demanding on computational time. It appears that each {\it assigner} training step requires $mN$ computations, where $m$ is the number of generated points and $N$ is the number points in the original dataset. Now if one want to avoid mode collapse, should make $m$ in a similar scale with $N$, which results to "perfect" training time being of order $O(N^{2}).$
\end{enumerate}

\section{Experiments} \label{ch:exeriments}
In the following section, we describe the layout for the experiments that were conducted. We will start by introducing the datasets and giving an intuition why these datasets are useful to compare the performance of different approaches. Afterwards, we will define metrics that can evaluate the performance of the results.
\subsection{Datasets}
In this section, we will describe the datasets that were used for the experiment.  We would like to note, that we did not include the standard by now Cifar10 dataset, because we were not able to reproduce the claimed results of the other papers on this dataset at all. We believe that this was solely due to our inability to properly produce the required architecture for the {\it generator}. However, we trained our model with a dense generator with a reasonably good outcome. The interested reader, can use the code in the repository to check for themselves.
\subsubsection{MNIST}

The MNIST dataset is a collection of handwritten digits from 0 to 9 and labels indicating the number it should represent as an integer. The individual images consist of 28x28  greyscale values  70000 images are part of the dataset. For the experiments, we reduce the number of images to 5000 examples upscale them to 32x32. The MNIST dataset is a typical dataset for comparing different GAN models because of their popularity in the machine learning community. Additionally, the numbers are easy to recognize and blurry images can be seen at a glance.
\subsubsection{Fashion-MNIST}
The fashion-MNIST dataset was introduced in \cite{DBLP:journals/corr/abs-1708-07747} as an alternative dataset to the MNIST. The fashion-MNIST is a collection of pictures of clothings. The number of classes is kept the same to  MNIST that represent ten different clothing types. The individual images consist of 28x28 greyscale values and the number of examples in the dataset is 70000. For our experiments, we reduce the number of images to 5000 examples and upscale them to 32x32.
We chose this dataset for multiple reasons, first we we can visually detect blur in the generated images by looking at the sharpness of transition between clothings and black background. Secondly we can see if details in the image get generated like prints on t-shirts, zippers and wrinkles. Third we think that its interesting to see difference between the results of MNIST compared to this dataset because their main dimensions and pixel ranges are the same.

\subsection{Metrics}

Evaluating the performance of generative models is a nontrivial task. \cite{theis2015note} discussed the evaluation of generative image models and concluded that no single evaluation metric is accurate, and that the right choice depends on the application at hand. Similarly, in \cite{Borji_2019}, the authors concluded, after reviewing 24 quantitative and 5 qualitative measures, that there is no universal measure between model performances. We decided to follow a different route and use Wasserstein-1  as a metric for evaluation, since it functions as an error function for the generator in the case of WGANs. Furthermore, we introduce a second metric that indirectly evaluates the appearance of mode collapse.
\subsubsection{Wasserstein-1 metric}
The Wasserstein-1 distance will now be used as a metric for evaluating the generated samples after training. By calculating the distance for a large number of samples from the generator and the dataset  we can approximate how close both distributions are to each other. For the experiments, we choose to sample ten times the amount of generated points compared to the size of the original dataset to get an accurate representation of the distance between our learned model and the dataset. To find the Wasserstein-1 metric we relied on the external library  POT: Python Optimal Transport (see \cite{flamary2017pot}) to solve the optimal transport problem that implements the algorithm proposed in ``Displacement interpolation using Lagrangian mass transport'' \cite{Bonneel:2011:DIU:2070781.2024192}.

\subsubsection{Assignment Variance }
Apart from the Wasserstein-1 metric to determine how well the model has been trained we propose a metric that is depended on the cost function that the model uses. We evaluate with this metric how well a particular model achieves an equal spread of generated points around each of the real points in the dataset. To do this we take ten times the number of generated points compared to the amount of real points in the dataset. We then use the model specific cost function $c$ and find the closest real point in the dataset. By counting how many generated points get assigned to a point in the real dataset we can determine if the model is spreading the points equally or if mode collapse is occurring in the model with respect to its cost $c$. To generate a single value we calculate the variance  around the perfect result of ten assignments.
\begin{equation}
\frac{1}{\#reals}*\sum_{i=1}^{\#reals}\sqrt{(\#assignments_i -10)^2}.
\end{equation}
\section{Results and Discussion}
In this section we will proceed with the results of our experiments and with our interpretation of them.
\subsubsection{MNIST}
\begin{figure}[H]
	\setlength{\extrarowheight}{.5em}
	\Large
	\centering
	\begin{tabular}[b]{ c c }
		GAN &  \hspace{32pt} Closest Real\\		
		\includegraphics[height=0.30\textwidth,width=0.30\textwidth]{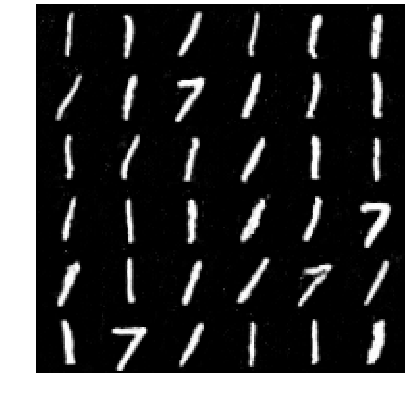} & \hspace{32pt}
		\includegraphics[height=0.30\textwidth,width=0.30\textwidth]{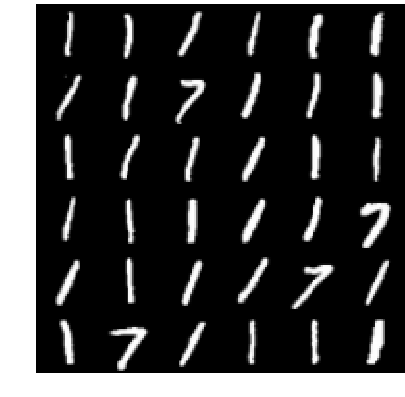} \\
		WGAN &  \hspace{32pt} Closest Real\\
		\includegraphics[height=0.30\textwidth,width=0.30\textwidth]{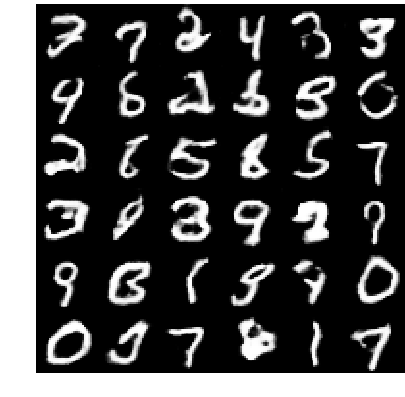} & \hspace{32pt}
		\includegraphics[height=0.30\textwidth,width=0.30\textwidth]{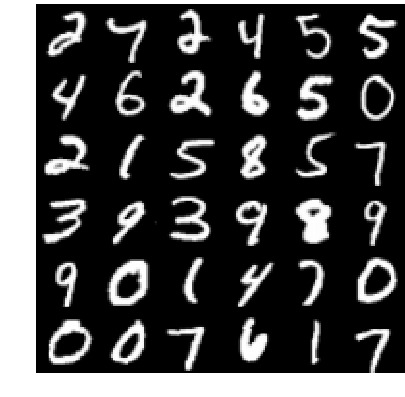} \\
		
		WGAN-GP &  \hspace{32pt} Closest Real \\	 
		\includegraphics[height=0.30\textwidth,width=0.30\textwidth]{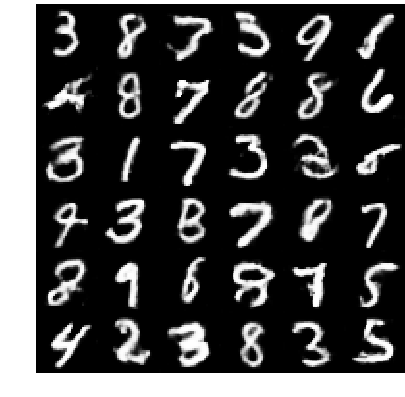} & \hspace{32pt}
		\includegraphics[height=0.30\textwidth,width=0.30\textwidth]{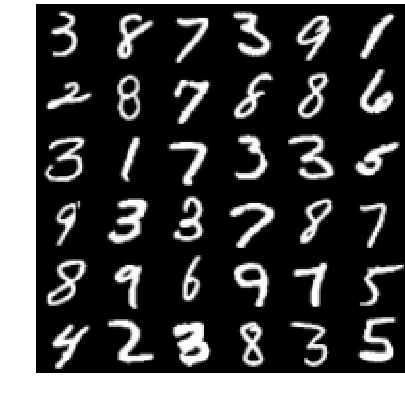} \\
		
	\end{tabular}	
	
	\caption{Generated samples for MNIST with GAN, WGAN and WGAN-GP. The closest real points were chosen by the cost function definition of the model. For vanilla GANs we used the Euclidean distance, since there is no cost function for this model.}\label{fig:MNIST}
\end{figure}
\begin{figure}[H]
	\centering
	\setlength{\extrarowheight}{.5em}
	\Large
	\begin{tabular}[b]{ c c }
		Square Assignment &  \hspace{32pt} Closest Real\\
		\includegraphics[height=0.30\textwidth,width=0.30\textwidth]{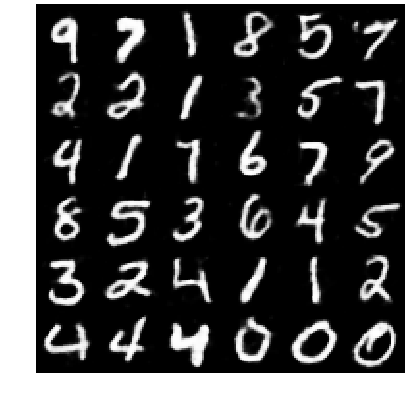} & \hspace{32pt}
		\includegraphics[height=0.30\textwidth,width=0.30\textwidth]{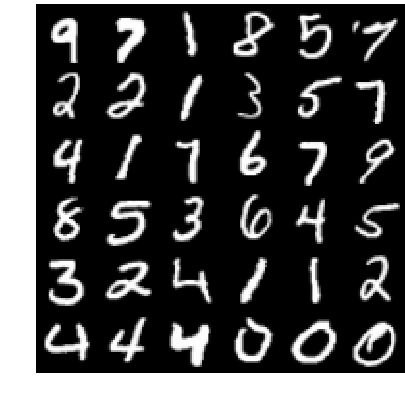} \\
		
		SSIM Assignment&  \hspace{32pt} Closest Real \\	
		\includegraphics[height=0.30\textwidth,width=0.30\textwidth]{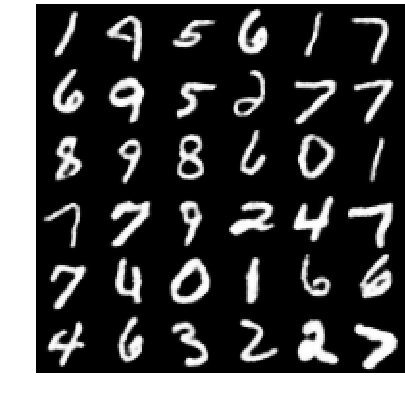} & \hspace{32pt}
		\includegraphics[height=0.30\textwidth,width=0.30\textwidth]{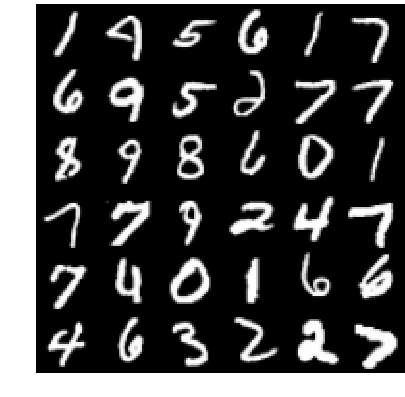} \\	
	\end{tabular}	
	
	\caption{Generated samples for MNIST with the assignment approach. The closest real points were chosen with respect to the cost function defined by the model.}\label{fig:MNIST_new}
\end{figure}

\begin{table}[h!]
	\begin{center}
		\begin{tabular}{c c c c c c }
			&  GAN  & WGAN  & WGAN-GP &  Square  & SSIM  \\
			Wasserstein Metric  & 15.80 & 12.59 & 10.69   & 9.68    &9.67 \\
			Assignment Variance &       & 0.46  & 0.29    & 0.14    & 0.21
		\end{tabular}
		\caption{Metrics applied on MNIST}\label{table:MNIST}
	\end{center}
\end{table}
 We generated samples  together with the closest real point based on the cost function of the method. Results from the models we use as comparison can be seen in figure \ref{fig:MNIST} and our results can be seen in \ref{fig:MNIST_new}. The original GAN approach seems to only produce two rather similar looking numbers while the other approaches produce a larger variety of numbers. The quality of the images seems similar to us while a little bit sharper for our squared and SSIM assignment. Additionally, we notice that WGAN, GAN-GP and the square assignment seem to produce numbers where pixels inside the lines of the numbers are missing or are not fully white while the SSIM assignment generated images are fully connected. We think this behavior reflects the fact that SSIM optimizes for the structural integrity of the number as well as equal luminance and contrast. The Wasserstein metric in Table \ref{table:MNIST} reflects the perceived image quality showing the best results for our method. Then looking at the numbers for the variance from optimal assignment we can see that our methods indeed managed to achieve their training objective the closest by having an equal spread around the dataset points.
In conclusion one can say that the experiments for MNIST show what we expected, the square distance is able to move the model closer to the real distribution by having a smoother gradient when points get close while the SSIM Assignment manages to produce perceptual more appealing images.

\subsubsection{Fashion MNIST}
\begin{figure}[H]
	\centering
	\setlength{\extrarowheight}{.5em}
	\Large
	\begin{tabular}[b]{c c}
		GAN & \hspace{32pt} Closest Real \\
		\includegraphics[height=0.30\textwidth,width=0.30\textwidth]{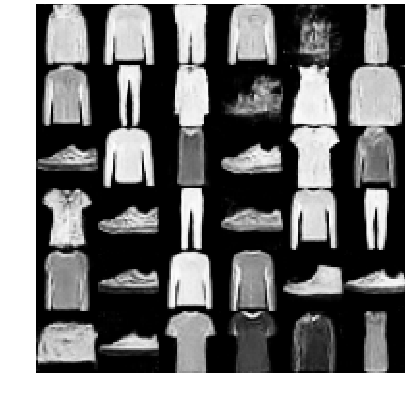} & 
		 \hspace{32pt}\includegraphics[height=0.30\textwidth,width=0.30\textwidth]{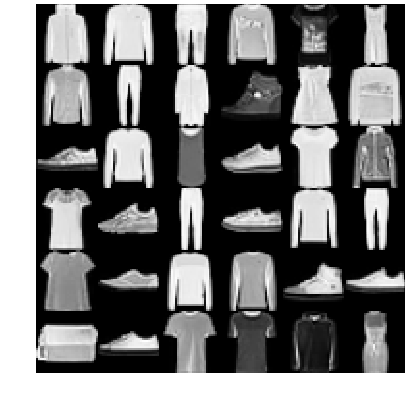} \\
		WGAN &  \hspace{32pt}Closest Real\\	
		\includegraphics[height=0.30\textwidth,width=0.30\textwidth]{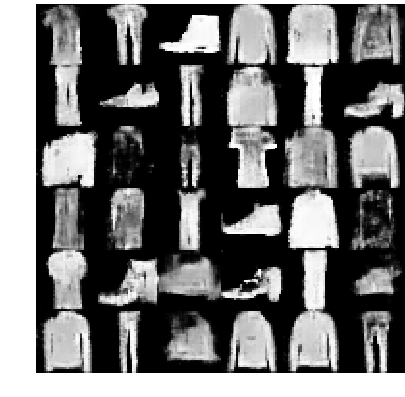} &  \hspace{32pt}
		\includegraphics[height=0.30\textwidth,width=0.30\textwidth]{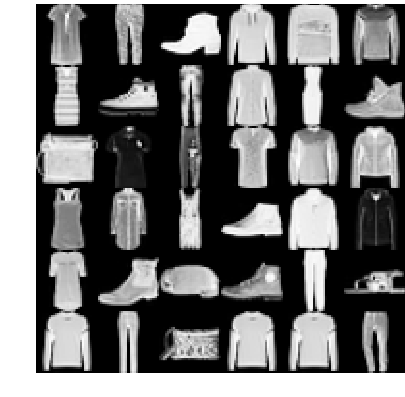} \\	WGAN-GP &  \hspace{32pt}Closest Real\\
		\includegraphics[height=0.30\textwidth,width=0.30\textwidth]{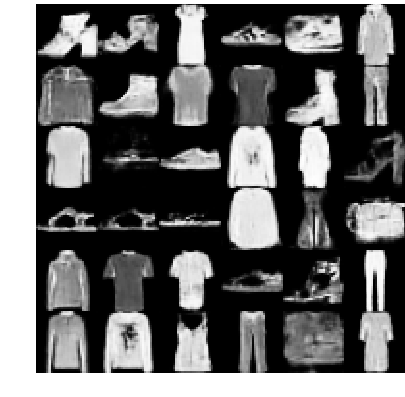} &  \hspace{32pt}
		\includegraphics[height=0.30\textwidth,width=0.30\textwidth]{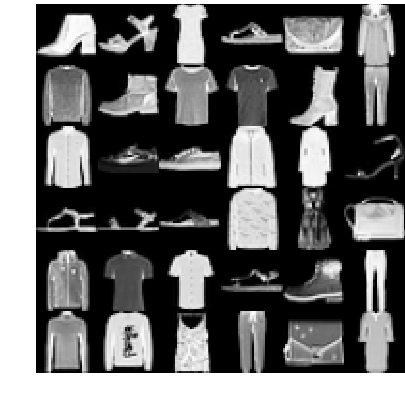} 
	\end{tabular}

	\caption{Generated samples for Fashion-MNIST with GAN, WGAN, and WGAN-GP.}\label{fig:fashion}
\end{figure}
\begin{figure}[H]
	\centering
	
	\setlength{\extrarowheight}{.5em}
	\Large
	\begin{tabular}[b]{c c}
		Square Assignment  & \hspace{32pt} Closest Real\\
		\includegraphics[height=0.30\textwidth,width=0.30\textwidth]{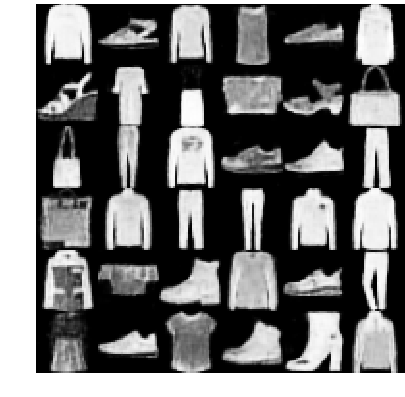} &  \hspace{32pt}
		\includegraphics[height=0.30\textwidth,width=0.30\textwidth]{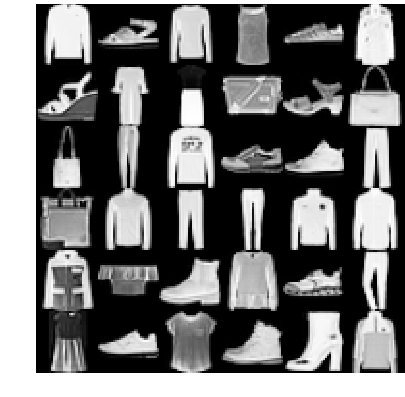} \\
		SSIM Assignment & \hspace{32pt} Closest Real\\	
		\includegraphics[height=0.30\textwidth,width=0.30\textwidth]{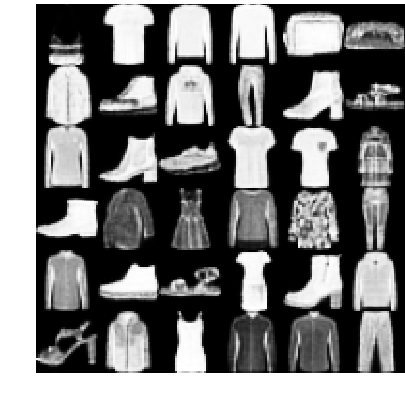} &  \hspace{32pt}
		\includegraphics[height=0.30\textwidth,width=0.30\textwidth]{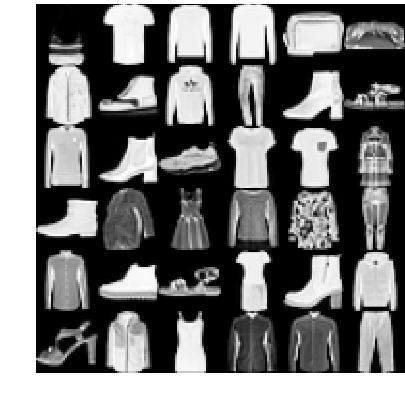} \\
	\end{tabular}	
	
	\caption{Generated samples for Fashion-MNIST for  square Assignment and SSIM assignment.}\label{fig:fashion_new}
\end{figure}
\begin{table}[h!]
	\begin{center}
		\begin{tabular}{c c c c c c c}
			&  GAN     & WGAN  & WGAN-GP &  Square & SSIM   \\
			Wasserstein Metric  & 11.03    & 11.19 & 9.15    & 3.40    & 7.17  \\
			Assignment Variance &          & 1.99  &0.82     & 0.01     & 0.05
		\end{tabular}
		\caption{Metrics for the Fashion-MNIST}\label{table:fashion}
		
	\end{center}
\end{table}

To judge  the visual quality, we can check the blurriness of the clothes at the edges of  transition from the background and additionally can see if smaller details are present like prints on t-shirts, zippers and wrinkles. The results in figure \ref{fig:fashion} show the output of the models for comparison and \ref{fig:fashion_new} the results of our approaches. The overall visual quality of the images seems to be equal except for WGAN showing artifacts and general blurriness. This might show that clipping the weights of the network  prevents it from learning the more complicated Fashion-MNIST distribution function. Additionally GAN seems to show no obvious mode collapse that is in strong contrast to the MNIST results. We interpret that discrepancy as an indicator that the original GAN is very sensitive to the choice of the hyperparameters, and we are in no way claiming that original GAN is not able to capture the MNIST dataset.  Another observation is that all GANs and even WGAN-GP, are producing points that seem to be far different from the real dataset but at the same time look realistic suggesting that it learned a model that found some underlying representation of the dataset. When looking at the details inside the clothings one can see that WGAN-GP as well as our approach produces some details but especially the SSIM assignment is reproducing a lot of the details. When examining the metrics in Table \ref{table:fashion} we can see that both of our approaches generate samples that are closer to the real distribution measured by the Wasserstein-1 distance. When looking at the variance from optimal assignment one can see that how well the approach respectively achieves its objective. For both of our assignment methods we can see that they generate an nearly equal amount of points around the real data points.

\subsection{Conclusion}
We explored the idea of training GNs with various optimal transport costs. As the first cost we choose the squared Euclidean norm, that should  have a smoother gradient for  points that lie close to each other. The second cost was the structural similarity index proposed by \cite{Wang:2004:IQA:2319031.2320551}, that tries to assess image quality by accounting for luminance, contrast and structure in the image that better reflect human perception of image quality. Choosing this cost allows us to train our generative model for generating images of high perceptual quality. To train generative models with these new costs a novel training procedure was introduced that allows us to use  these costs but they can be exchanged for any metric. The downside is that the computational effort for a training step increases superlinear with respect to the number of points in the dataset. Our experiments show that it is indeed possible to train a model with our proposed cost and decrease the distance between the generated points and the points in the dataset compared to approaches by other authors. Furthermore, the experiment results with the SSIM cost shows that the cost indeed influences the appearance of the generated images. The Wasserstein-1 metric applied on the results, further shows that we are able to better approximate the distribution of both MNIST as well Fashion-MNIST datasets when using the assigning approach.

Several new research directions can build upon our work. There can be follow up work that tries to tackle the computational burden introduced by our approach. We think that the largest improvement can be achieved by exploring more efficient ways to find the closest neighbours between the real and generated points that are necessary for our algorithm. Secondly on can try new costs that better reflect the training objective one tries to achieve. Especially with datasets where the Euclidean distance is a poor choice for assessing the similarity between data points. At last, one can look at the architecture choices made for the neural networks, the optimizers and the hyperparameters to either optimize them in general or for a given dataset and training objective. 

Our repository for the experiment can be found in https://github.com/artnoage/Optimal-Transport-GAN.

\section*{Appendix}
\begin{theorem}\label{one}
Let $\psi_{w}:\X\rightarrow\mathbb{R},$ a collection of functions parameterized by $w.$ Let also assume that for every $y\in\X, \psi_{w}(y)$ is a continuous function with respect to $w,$ and that $c$ satisfies the Assumption \ref{assumption}. Finally, let $\mathbb{P}$ be a distribution in   $\X,$ that is absolutely continuous with respect to the Lebesgue measure. 

For $w_{0}$ and a finite set $Y,$ we have
that for almost every $x\in\X,$ the expression
$$ \psi_{w_{0}}(y) + c(x,y)$$ has a unique minimizer in $Y.$ Evermore with probability one, for every independent random sample, with respect to $\mathbb{P}$ of points $\{x_{1},..,x_{n},\dots\}$ in $X$, we have that it exists $\delta(\{x_{1},\dots,x_{n}\}),$ such that
\begin{equation}
\arg\min_{Y}\{\psi_{w_{0}}(y) + c(x_{i},y)\}=\arg\min_{Y}\{\psi_{w}(y) + c(x_{i},y)\},\hspace{10pt} \forall w\in B(w_{0},\delta(\{x_{1},\dots,x_{n}\})).
\end{equation}
\end{theorem}
\begin{proof}
	Let assume that the set $\widetilde{\X},$ of all points that have multiple minimizers has positive Lebesgue measure. Then, since $Y$ is finite, there exists at least one $y_{0}\in Y$ such that $\widetilde{\X}_{y_{0}}$ of the points that are minimizers for $$ \psi_{w_{0}}(y_{0}) + c(x,y_{0})$$ has positive Lebesgue measure. From that, we can induce that the set $$\{x\in\widetilde{\X}:c(x,y_{0})= \{\psi_{w_{0}}(y_{0}) + c(x_{0},y_{0})\} -\psi_{w_{0}}(y_{0})= c(x_{0},y_{0})\},$$ where $x_{0}$ is an arbitrary but fixed point in $\widetilde{\X}_{y_{0}},$ has positive measure. contradicts the assumption \ref{assumption}.
	
	Now since the points $x_{1},\dots,x_{n},\dots$ are sampled independently from $\mathbb{P},$ with probability one, the expression $\psi_{w_{0}}(y) + c(x_{i},y),$
	has unique minimizer $y(x_{i}),$ and further more it exists $\epsilon(x_{i}),$ such that  
	\begin{equation}
	\psi_{w_{0}}(y) + c(x_{i},y)>\psi_{w_{0}}(y(x_{i})) + c(x_{i},y(x_{i}))+\epsilon(x_{i}),\hspace{16pt} \forall y\in Y\setminus\{y(x_{i})\}.
	\end{equation}
Now if we pick $\delta(\{x_{1},\dots,x_{n}\})$ such that 
$$|f_{w}(y_{i})-f_{w_{0}}(y_{i})|<\min_{x_{i}}\epsilon(x_{i}),$$ we get the result.
\end{proof}

We conclude with the proof of Theorem \ref{minith}.

\begin{proof}[Proof of Theorem \ref{minith}]
Let $Q$ be an optimal plan between $\mu$ and $\nu.$ By definition of $\widetilde{T},$
we have
$$\psi(\widetilde{T}(x))+c(x,\widetilde{T}(x))=\psi^{supp(\nu),c}(x)=\inf_{y\in supp(\nu)}\{\psi(y)+c(x,y)\}\leq \psi(y)+c(x,y) \hspace{16pt} \forall y\in supp(\nu).$$
By integrating with respect to $Q,$ we have
$$\int \psi(\widetilde{T}(x))+c(x,\widetilde{T}(x))dQ \leq \int \left(c(x,y)+\psi(y)\right)dQ,$$
which gives
$$\int\psi(\widetilde{T}(x))d\mu+\int c(x,\widetilde{T}(x))d\mu\leq \int\psi(y)d\nu+\int c(x,y)dQ.$$
Since $\widetilde{T}_{\sharp}\mu=\nu,$ the last inequality gives us
$$\int c(x,\widetilde{T}(x))d\mu\leq\int c(x,y)dQ,$$
which proves that $\widetilde{T}$ is an optimal map.
\end{proof}
\bibliographystyle{plain}
\bibliography{wgan,literatur,bib}
\end{document}